\documentclass[letterpaper, 10 pt, conference]{ieeeconf} 
\pdfoutput=1
\usepackage{geometry}        
\usepackage{graphicx}
\usepackage{amsmath}
\usepackage{caption}
\usepackage{subcaption}
\IEEEoverridecommandlockouts                              
\newtheorem{theorem}{Theorem}

\title{\LARGE \bf Baidu Apollo EM Motion Planner}
\author{Haoyang Fan\(^{1,\dag}\), Fan Zhu\(^{2,\dag}\), Changchun Liu, Liangliang Zhang, Li Zhuang, \\Dong Li, Weicheng Zhu, Jiangtao Hu, Hongye Li, Qi Kong\(^{3,*}\) 
\thanks{\(^{\dag}\) Authors who contributed equally in this manuscript}%
\thanks{\(^{*}\) Corresponding author for this manuscript}%
\thanks{\(^{\#}\) www.apollo.auto}%
\thanks{\(^{1}\) Haoyang Fan is with Baidu USA LLC, fanhaoyang01@baidu.com.}
\thanks{\(^{2}\) Fan Zhu is with Baidu USA LLC, fanzhu@baidu.com.}
\thanks{\(^{3}\) Qi Kong is with Baidu USA LLC, kongqi02@baidu.com.}
}

\begin{document}

\maketitle
\thispagestyle{empty}
\pagestyle{empty}

\begin{abstract}

In this manuscript, we introduce a real-time motion planning system based on the Baidu Apollo (open source) autonomous driving platform. 
The developed system aims to address the industrial level-4 motion planning problem while considering safety, comfort and scalability. 
The system covers multilane and single-lane autonomous driving in a hierarchical manner:
(1) The top layer of the system is a multilane strategy that handles lane-change scenarios by comparing lane-level trajectories computed in parallel.
(2) Inside the lane-level trajectory generator, it iteratively solves path and speed optimization based on a Frenet frame.
(3) For path and speed optimization, a combination of dynamic programming and spline-based quadratic programming is proposed to construct a scalable and easy-to-tune framework to handle traffic rules, obstacle decisions and smoothness simultaneously.
The planner is scalable to both highway and lower-speed city driving scenarios. We also demonstrate the algorithm through scenario illustrations and on-road test results. 

The system described in this manuscript has been deployed to dozens of Baidu Apollo autonomous driving vehicles since Apollo v1.5 was announced in September 2017. 
As of May 16th, 2018, the system has been tested under 3,380 hours and approximately 68,000 kilometers (42,253 miles) of closed-loop autonomous driving under various urban scenarios.

The algorithm described in this manuscript is available at 
\url{https://github.com/ApolloAuto/apollo/tree/master/modules/planning}.
\end{abstract}

\section{Introduction}

Autonomous driving research began in the 1980s and has significantly grown over the past ten years.
Autonomous driving aims to reduce road fatalities, increase traffic efficiency and provide convenient travel. However, autonomous driving is a challenging task that requires accurately sensing the environment, a deep understanding of vehicle intentions and safe driving under different scenarios. To address these difficulties, we constructed an Apollo open source autonomous driving platform. 
The flexible modularized architecture of the developed platform supports fully autonomous driving deployment \url{https://github.com/ApolloAuto/apollo}.

In the figure, the HD map module provides a high-definition map that can be accessed by every on-line module. Perception and localization modules provide the necessary dynamic environment information, which can be further used to predict future environment status in the prediction module. The motion planning module considers all information to generate a safe and smooth trajectory to feed into the vehicle control module. 

In motion planner, safety is always the top priority. 
We consider autonomous driving safety in, but not limited to, the following aspects: traffic regulations, range coverage, cycle time efficiency and emergency safety. All these aspects are critical. Traffic regulations are designed by governments for public transportation safety, and such regulations also apply to autonomous driving vehicles. An autonomous driving vehicle should follow traffic regulations at all times. For range coverage, we aim to provide a trajectory with at least an eight second or two hundred meter motion planning trajectory. The reason is to leave enough room to maintain safe driving within regular autonomous driving vehicle dynamics. The execution time of the motion planning algorithm is also important. 
In the case of an emergency, the system could react within 100 ms, compared with a 300 ms reaction time for a normal human driver. 
A safety emergency module is the last shell for protecting the safety of riders. 
For a level-4 motion planner, once upstream modules can no longer function normally, the safety module within the motion planner shall respond with an immediate emergency behavior and send warnings that interrupt humans. 
Moreover, the autonomous driving system shall have the ability to react to emergencies in a lower-level-like control module. Further safety design is beyond this manuscript's scope.

Fig. \ref{fig:apollo_structure} shows the architecture of the Apollo online modules. 
\begin{figure}[htbp!]
\begin{center}
\includegraphics[width = 0.5 \textwidth]{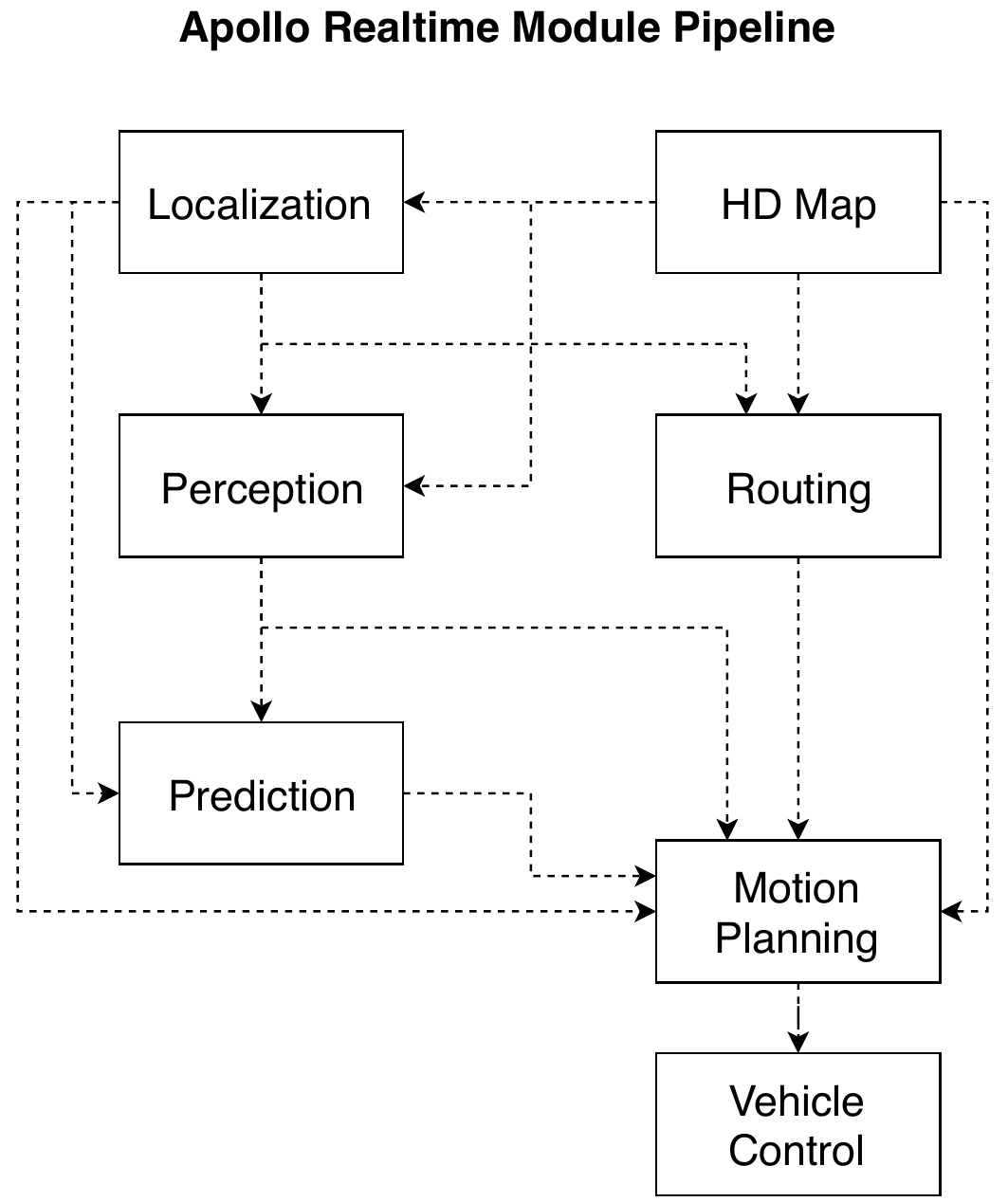}
\caption{On board modules of the Apollo open source autonomous driving platform}
\label{fig:apollo_structure}
\end{center}
\end{figure}

In addition to safety, passengers' ride experience is also important. The measurement of ride experience includes, but is not limited to, scenario coverage, traffic regulation and comfort. For scenario coverage, the motion planner should not only be able to handle simple driving scenarios (e.g., stop, nudge, yield and overtake) but also handle multilane driving, heavy traffic and other complicated on-road driving scenarios. 
Planning within traffic regulations is also important for ride experience. In addition to being a safety requirement, following the traffic regulations will also minimize the risk of accidents and reduce emergency reactions for autonomous driving. The comfort during autonomous driving is also important. In motion planning, comfort is generally measured by the smoothness of the provided autonomous driving trajectory.

This manuscript presents the Apollo EM planner, which is based on an EM-type iterative algorithm (\cite{dempster1977maximum}).
This planner targets safety and ride experience with a multilane, path-speed iterative, traffic rule and decision combined design. The intuitions for this planner are introduced as follows:

\subsection{Multilane Strategy}
For level-4 on-road autonomous driving motion planning, a lane-change strategy is necessary. 
One common approach is to develop a search algorithm with a cost functional on all possible lanes and then select a final trajectory from all the candidates with the lowest cost  \cite{ajanovic2018search}, \cite{werling2010optimal}.  
 This approach has some difficulties.
First, the search space is expanded across multiple lanes, which causes the algorithm to be computationally expensive.
Second, traffic regulations (e.g., right of the road, traffic lights) are different across lanes, and it is not easy to apply traffic regulations under the same frame. 
Furthermore, trajectory stability that avoids sudden changes between cycles should be taken into consideration. It is important to follow consistent on-road driving behavior to inform other drivers of the intention of the autonomous driving vehicle.

Typically, a multilane strategy should cover both nonpassive and passive lane-change scenarios. 
In EM planner, a nonpassive lane-change is a request triggered by the routing module for the purpose of reaching the final destination.
A passive lane change is defined as an ego car maneuver when the default lane is blocked by the dynamic environment. In both passive and nonpassive lane changes,  we aim to deliver a safe and smooth lane-change strategy with a high success rate. 
Thus, we propose a parallel framework to handle both passive and nonpassive lane changes. 
For candidate lanes, all obstacles and environment information are projected on lane-based Frenet frames. Then, the traffic regulations are bound with the given lane-level strategy.
Under this framework, each candidate lane will generate a best-possible trajectory based on the lane-level optimizer. Finally, a cross-lane trajectory decider will determine which lane to choose based on both the cost functional and safety rules.

\subsection{Path-Speed Iterative Algorithm}
In lane-level motion planning, optimality and time consumption are both important. 
Thus, many autonomous driving motion planning algorithms are developed in Frenet frames with time (SLT) to reduce the planning dimension with the help of a reference line.
Finding the optimal trajectory in a Frenet frame is essentially a 3D constrained optimization problem. There are typically two types of approaches: direct 3D optimization methods and the path-speed decoupled method. 
Direct methods (e.g., \cite{mcnaughton2011motion} and \cite{ziegler2009spatiotemporal}) attempt to find the optimal trajectory within SLT using either trajectory sampling or lattice search. These approaches are limited by their search complexity, which increases as both the spatial and temporal search resolutions increase. 
To qualify the time consumption requirement, one has to compromise with increasing the search grid size or sampling resolution. Thus, the generated trajectory is suboptimal. 
Conversely, the path-speed decoupled approach optimizes path and speed separately. Path optimization typically considers static obstacles. Then, the speed profile is created based on the generated path \cite{gu2015tunable}. It is possible that the path-speed approach is not optimal with the appearance of dynamic obstacles. 
However, since the path and speed are decoupled, this approach achieves more flexibility in both path and speed optimization. 

EM planner optimizes path and speed iteratively. The speed profile from the last cycle is used to estimate interactions with oncoming and low-speed dynamic obstacles in the path optimizer. Then, the generated path is sent to the speed optimizer to evaluate an optimal speed profile. For high-speed dynamic obstacles, EM planner prefers a lane-change maneuver rather than nudging for safety reasons. Thus, the iterative strategy of EM planner can help to address dynamic obstacles under the path-speed decoupled framework.

\subsection{Decisions and Traffic Regulations}

In EM planner, decisions and traffic regulations are two different constraints. 
Traffic regulations are a non-negotiable hard constraint, whereas obstacle yield, overtake, and nudge decisions are negotiable based on different scenarios. 
For the decision-making module, some planners directly apply numerical optimization \cite{xu2012real}, \cite{ziegler2009spatiotemporal} to make decisions and plans simultaneously. 
In Apollo EM planner, we make decisions prior to providing a smooth trajectory. 
The decision process is designed to make on-road intentions clear and reduce the search space for finding the optimal trajectory. Many decision-included planners attempt to generate vehicle states as the ego car decision. These approaches can be further divided into hand-tuning decisions (\cite{montemerlo2008junior}, \cite{urmson2008autonomous}, \cite{werling2010optimal}) and model-based decisions (\cite{bai2014integrated}, \cite{brechtel2014probabilistic}). 
The advantage of hand-tuning decision is its tunability. 
However, scalability is its limitation. In some cases, scenarios can go beyond the hand-tuning decision rule's description. 
Conversely, the model-based decision approaches generally discretize the ego car status into finite driving statuses and use data-driven methods to tune the model. In particular, some papers, such as \cite{cunningham2015mpdm} and \cite{galceran2015multipolicy}, propose a unified framework to handle decisions and obstacle prediction simultaneously. The consideration of multi-agent interactions will benefit both prediction and decision-making processes.

Targeting level-4 autonomous driving, a decision module shall include both scalability and feasibility. Scalability is the scenario expression ability (i.e., the autonomous driving cases that can be explained). 
When considering dozens of obstacles, the decision behavior is difficult to be accurately described by a finite set of ego car states.
For feasibility, we mean that the generated decision shall include a feasible region in which the ego car can maneuver within dynamic limitations. 
However, both hand-tuning and model-based decisions do not generate a collision-free trajectory to verify the feasibility. 

In EM planner's decision step, we describe the behavior differently. First, the ego car moving intention is described by a rough and feasible trajectory. Then, the interactions between obstacles are measured with this trajectory. This feasible trajectory-based decision is scalable even when scenarios become more complicated. 
Second, the planner will also generate a convex feasible space for smoothing spline parameters based on the trajectory. 
A quadratic-programming-based smoothing spline solver could be used to generate smoother path and speed profiles that follow the decision. This guarantees a feasible and smooth solution.

The remainder of this paper is organized as follows. In section \ref{sec:multi-lane}, we introduce the multilane framework inside EM planner. In section \ref{sec:lane-level}, we focus on lane-level optimization and discuss EM iteration step by step. In section \ref{sec:example}, we provide an example with oncoming traffic for a simple demonstration. Section \ref{sec:performance} discuss the performance of EM planner. In section \ref{sec:conclusion}, we finalize the discussion and present conclusions for EM planner.

\section{EM Planner Framework with Multilane Strategy}\label{sec:multi-lane}

\begin{figure}[htbp]
\begin{center}
\includegraphics[width = 0.45\textwidth]{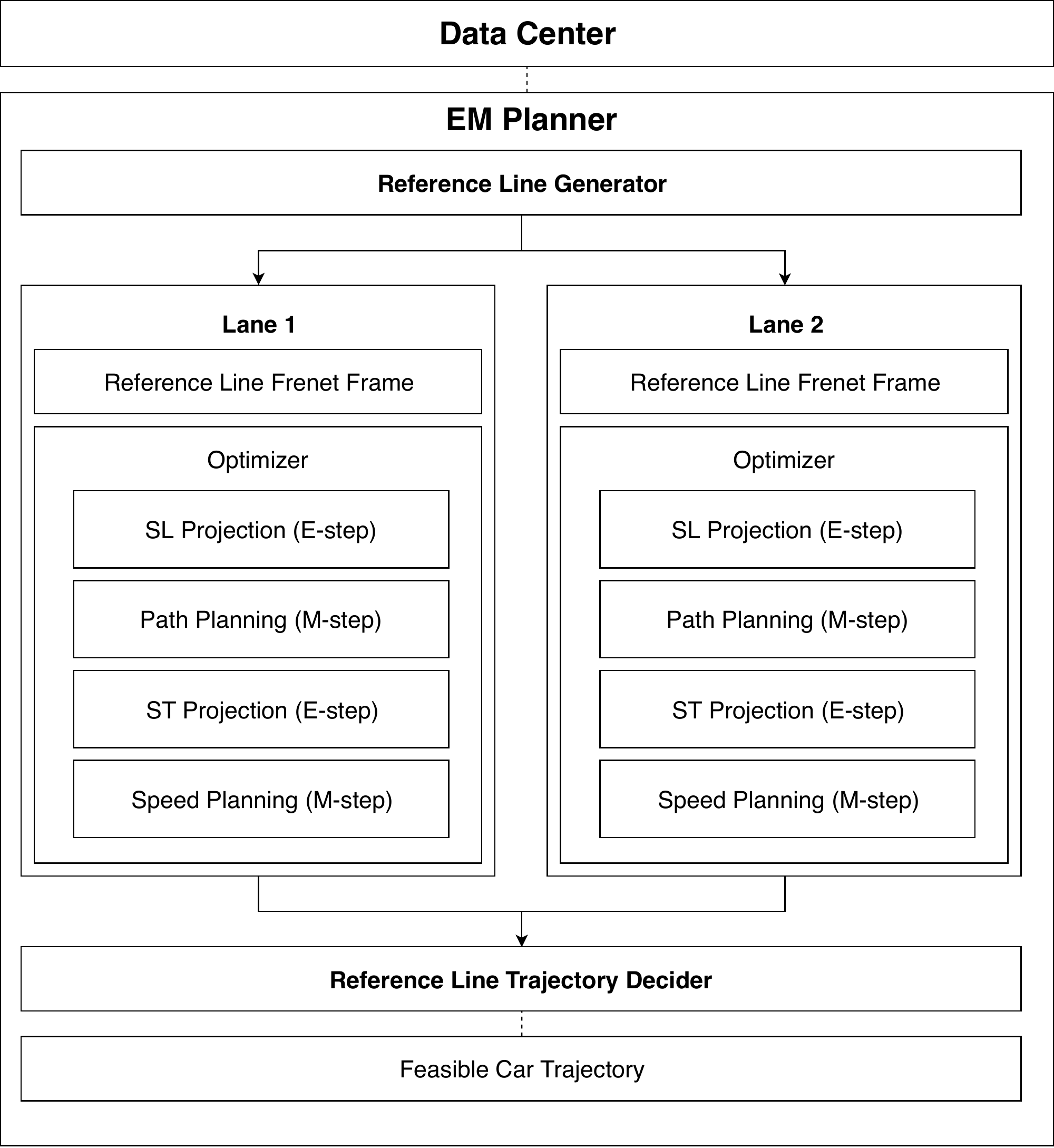}
\caption{EM Framework}
\captionsetup[figure]{labelfont={sc},textfont=normalfont,singlelinecheck=on,justification=centered,labelsep=colon}
\label{fig:em_structure}
\end{center}
\end{figure}

In this section, we first introduce the architecture of Apollo EM planner, and then we focus on the structure of the lane-level optimizer. Fig. \ref{fig:em_structure} presents an overview of EM planner. On top of the planner, all sources of information are collected and synced at the data center module. After data collection, the reference line generator will produce some candidate lane-level reference lines along with information about traffic regulations and obstacles. This process is based on the high-definition map and navigation information from the routing module. During lane-level motion planning, we first construct a Frenet frame based on a specified reference line. The relation between the ego car and its surrounding environment is evaluated in the Frenet frame constructed by the reference line, as well as traffic regulations. Furthermore, restructured information passes to the lane-level optimizer. The lane-level optimizer module performs path optimization and speed optimization. During path optimization, information about the surroundings is projected on the Frenet frame (E-step).
 Based on the information projected in the Frenet frame, a smooth path is generated (M-step). Similarly, during speed optimization, once a smooth path is generated by the path optimizer, obstacles are projected on the station-time graph (E-step). Then, the speed optimizer will generate a smooth speed profile (M-step). Combining path and speed profiles, we will obtain a smooth trajectory for the specified lane. In the last step, all lane-level best trajectories are sent to the reference line trajectory decider. Based on the current car status, regulations and the cost of each trajectory, the trajectory decider will decide a best trajectory for the ego car maneuver.

\section{EM Planner at Lane Level}\label{sec:lane-level}
In this section, we discuss the lane-level optimization problem. Fig. \ref{fig:em_iter} shows the path-speed EM iteration inside lane-level planning. The iteration includes two E-steps and two M-steps in one planning cycle. The trajectory information will iterate between planning cycles. We explain the submodules as follows.

\begin{figure}[htbp]
\begin{center}
\includegraphics[width = 0.45\textwidth]{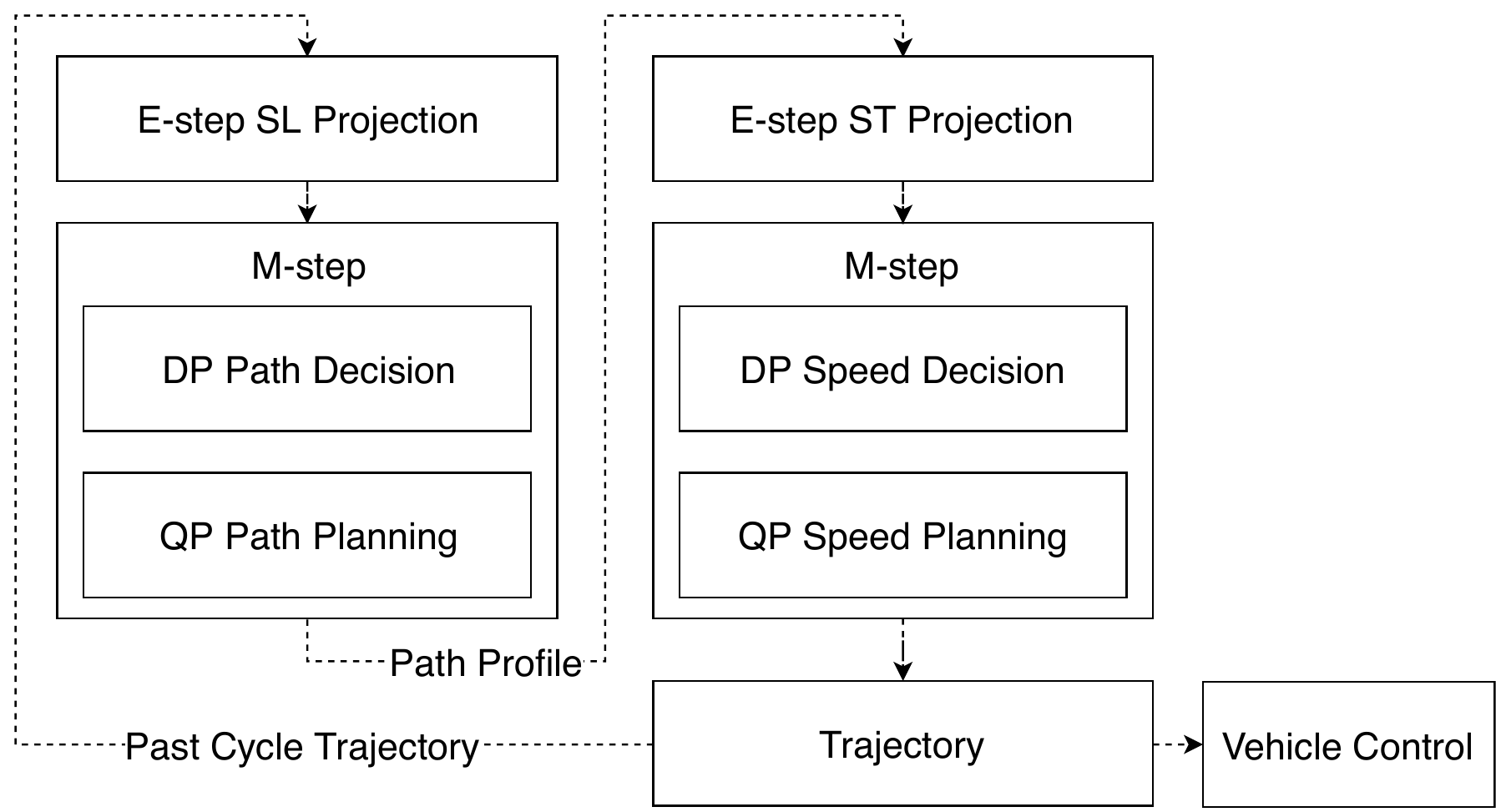}
\captionsetup[figure]{labelfont={sc},textfont=normalfont,singlelinecheck=on,justification=centered,labelsep=colon}\caption{EM Iteration}
\label{fig:em_iter}
\end{center}
\end{figure}

In the first E-step, obstacles are projected on the lane Frenet frame. This projection includes both static obstacle projection and dynamic obstacle projection. Static obstacles will be projected directly based on a Cartesian-Frenet frame transformation. In the Apollo framework, the intentions of dynamic obstacles are described with an obstacle moving trajectory. Considering the previous cycle planning trajectory, we can evaluate the estimated dynamic obstacle and ego car positions at each time point. Then, the overlap of dynamic obstacles and the ego car at each time point will be mapped in the Frenet frame. In addition, the appearance of dynamic obstacles during path optimization will eventually lead to nudging. Thus, for safety considerations, the SL projection of dynamic obstacles will only consider low-speed traffic and oncoming obstacles. For high-speed traffic, EM planner's parallel lane-change strategy will cover the scenario.
In the second E-step, all obstacles, including high-speed, low-speed and oncoming obstacles, are evaluated on the station-time frame based on the generated path profile. If the obstacle trajectory has overlap with the planned path, then a corresponding region in the station-time frame will be generated.

In two M-steps, path and speed profiles are generated by a combination of dynamic programming and quadratic programming. Although we projected obstacles on SL and ST frames, the optimal path and speed solution still lies in a non-convex space. Thus, we use dynamic programming to first obtain a rough solution; meanwhile, this solution can provide obstacle decisions such as nudge, yield and overtake. We use the rough decision to determine a convex hull for the quadratic-programming-based spline optimizer. Then, the optimizer can find solutions within the convex hull. We will cover the modules in the following.

\subsection{SL and ST Mapping (E-step)}

The SL projection is based on a G2 (continuous curvature derivative) smooth reference line as in \cite{werling2010optimal}. In Cartesian space, obstacles and the ego car status are described with location and heading $(x, y, \theta)$, as well as curvature and the derivative of curvature $(\kappa, d\kappa)$ for the ego car. Then, these are mapped to the Frenet frame coordinates $(s, l, dl, ddl, dddl)$, which represent station, lateral, and lateral derivatives. Since the positions of static obstacles are time invariant, the mapping is straightforward. For dynamic obstacles, we mapped the obstacles with the help of the last cycle trajectory of the ego car. The last cycle's moving trajectory is projected on the Frenet frame to extract the station direction speed profile.  This will provide an estimate of the ego car's station coordinates given a specific time. The estimated ego car station coordinates will help to evaluate the dynamic obstacle interactions. Once an ego car's station coordinates have interacted with an obstacle trajectory point with the same time, a shaded area on the SL map will be marked as the estimated interaction with the dynamic obstacle. Here, the interaction is defined as the ego car and obstacle bounding box overlapping. For example, as shown in Fig.~\ref{fig:sl_projection}, an oncoming dynamic obstacle and corresponding trajectory estimated from the prediction module are marked in red. The ego car is marked in blue. The trajectory of the oncoming dynamic obstacle is first discretized into several trajectory points with time, and then the points are projected to the Frenet frame. Once we find that the ego car's station coordinates have an interaction with the projected obstacle points, the overlap region (shown in purple in the figure) will be marked in the Frenet frame.

\begin{figure*}[htbp]
\begin{center}
\includegraphics[width = 1.0\textwidth]{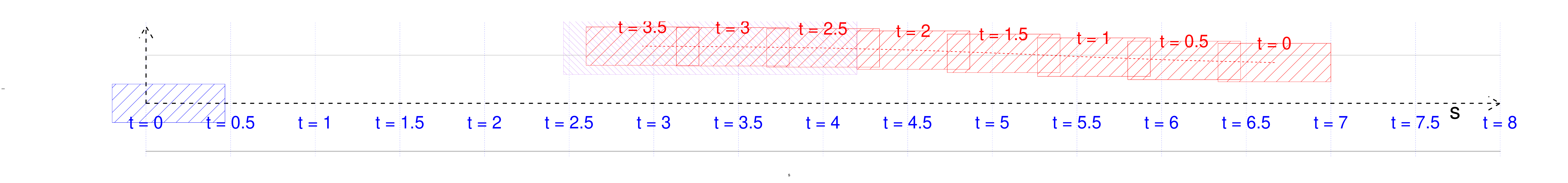}
\captionsetup[figure]{labelfont={sc},textfont=normalfont,singlelinecheck=on,justification=centered,labelsep=colon}
\caption{SL projection with oncoming traffic example}
\label{fig:sl_projection}
\end{center}
\end{figure*}

ST projection helps us evaluate the ego car's speed profile. After the path optimizer generates a smooth path profile in the Frenet Frame, both static obstacle and dynamic obstacle trajectories are projected on the path if there are any interactions.
 An interaction is also defined as the bounding boxes overlapping. In Fig.~\ref{fig:st_graph}, one obstacle cut into the ego driving path at 2 seconds and 40 meters ahead, as marked in red, and one obstacle behind the ego car is marked in green. The remaining region is the speed profile feasible region. The speed optimization M-step will attempt to find a feasible smooth solution in this region.

\begin{figure}[htbp]
\begin{center}
\includegraphics[width = 0.4\textwidth]{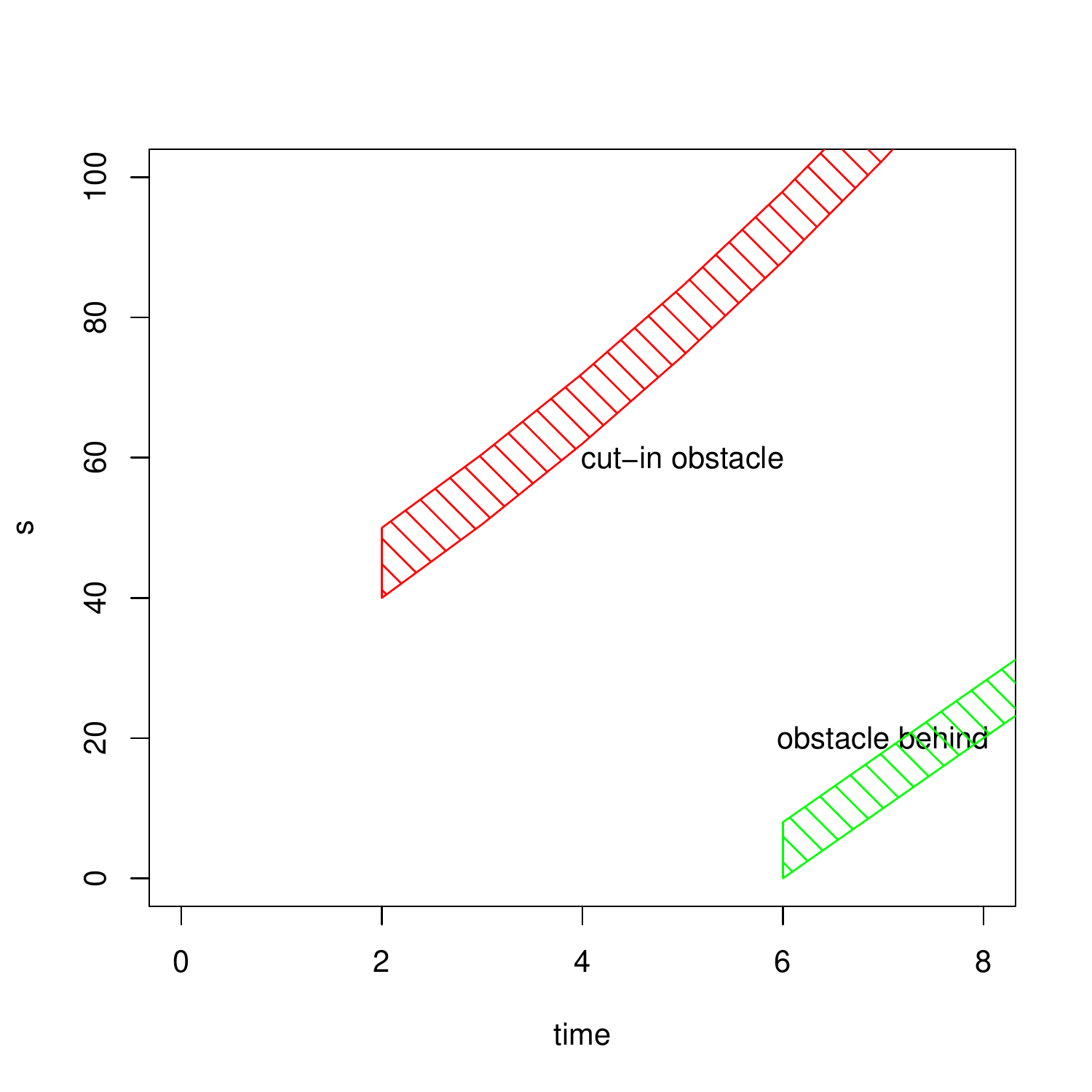}
\captionsetup[figure]{labelfont={sc},textfont=normalfont,singlelinecheck=on,justification=centered,labelsep=colon}
\caption{ST projection with cut-in obstacle and obstacle behind ego car}
\label{fig:st_graph}
\end{center}
\end{figure}

\subsection{M-Step DP Path}
The M-step path optimizer optimizes the path profile in the Frenet frame. This is represented as finding an optimal function of lateral coordinate $l = f(s)$ w.r.t. station coordinate in nonconvex SL space (e.g., nudging from left and right might be two local optima). Thus, the path optimizer includes two steps: dynamic-programming-based path decision and spline-based path planning. The dynamic programming path step provides a rough path profile with feasible tunnels and obstacle nudge decisions. As shown in Fig.~\ref{fig:dp_path_structure}, the step includes a lattice sampler,  cost function and dynamic programming search. 

The lattice sampler is based on a Frenet frame. As shown in Fig.~\ref{fig:dp_path}, multiple rows of points are first sampled ahead of the ego vehicle. Points between different rows are smoothly connected by quintic polynomial edges. The interval distance between rows of points depends on the speed, road structure, lane change and so forth. The framework allows customizing the sampling strategy based on application scenarios. For example, a lane change might need a longer sampling interval than current lane driving. In addition, the lattice total station distance will cover at least 8 seconds or 200 meters for safety considerations.

After the lattice is constructed, each graph edge is evaluated by the summation of cost functionals. We use information from the SL projection, traffic regulations and vehicle dynamics to construct the functional. The total edge cost functional is a linear combination of smoothness, obstacle avoidance and lane cost functionals. 
\[C_{total}(f(s)) = C_{smooth} (f) + C_{obs} (f) + C_{guidance} (f) \]

The smoothness functional for a given path is measured by:
\begin{align*} C_{smooth}(f) &= w_1 \int (f'(s))^2 ds + w_2 \int (f''(s))^2 ds\\
&+ w_3 \int (f'''(s))^2 ds.\end{align*}
In the smoothness cost functional, $f'(s)$ represents the heading difference between the lane and ego car, $f''(s)$ is related to the curvature of the path, and $f'''(s)$ is related to the derivative of the curvature of the ego car. With the form of polynomials, the above cost can also be evaluated analytically.

The obstacle cost given an edge is evaluated at a sequence of fixed station coordinates $\{s_0, s_1, ..., s_n\}$ with all obstacles. The obstacle cost functional is based on the bounding box distance between the obstacle and ego car. Denote the distance as $d$. The form of individual cost is given by:

\[C_{obs}(d) =  \begin{cases} 0, &d > d_n \\
						   C_{nudge} (d - d_c), &d_c \leq d \leq d_n  \\
						   C_{collision} & d < d_c\end{cases},\]
where $C_{nudge}$ is defined as a monotonically decreasing function. $d_c$ is set to leave a buffer for safety considerations. The nudge range $d_n$ is negotiable based on the scenario. $C_{collision}$ is the collision cost, which has a large value that helps to detect infeasible paths.

The lane cost includes two parts: guidance line cost and on-road cost. The guidance line is defined as an ideal driving path when there are no surrounding obstacles. This line is generally extracted as the centerline of the path. Define the guidance line function as $g(s)$. Then, it is measured as
\[C_{guidance}(f) = \int (f(s) - g(s))^2 ds\].
The on-road cost is typically determined by the road boundary. Path points that are outside the road will have a high penalty.

The final path cost is a combination of all these smooth, obstacle and lane costs.
Then, edge costs are used to select a candidate path with the lowest cost through a dynamic programming search.
The candidate path will also determine the obstacle decisions. For example, in Fig.~\ref{fig:dp_path}, the obstacle is marked as a nudge from the right side. 

\begin{figure}[htbp]
\begin{center}
\includegraphics[width = 0.45\textwidth]{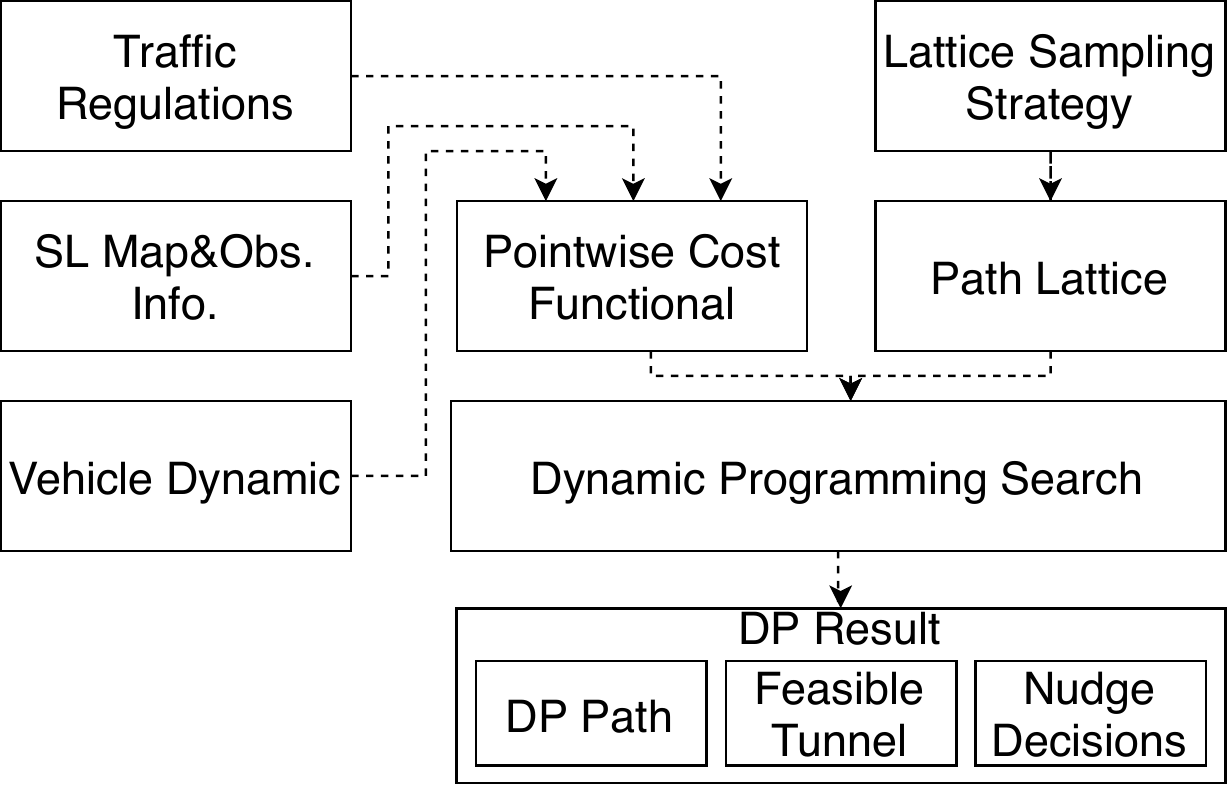}
\captionsetup[figure]{labelfont={sc},textfont=normalfont,singlelinecheck=on,justification=centered,labelsep=colon}
\caption{Dynamic programming structure}
\label{fig:dp_path_structure}
\end{center}
\end{figure}

\begin{figure*}[htbp]
\begin{center}
\includegraphics[width = 1.0\textwidth]{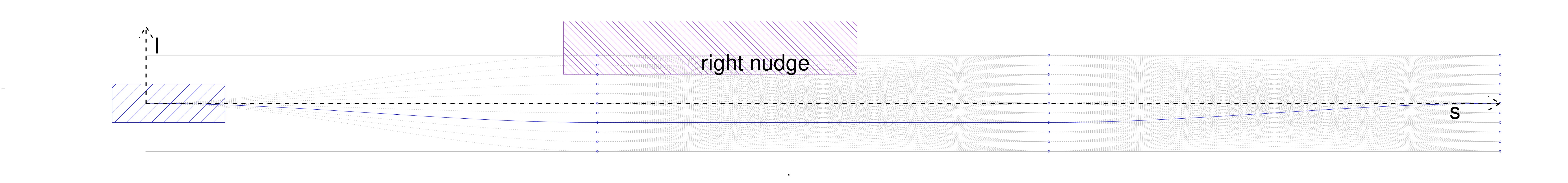}
\includegraphics[width = 1.0\textwidth]{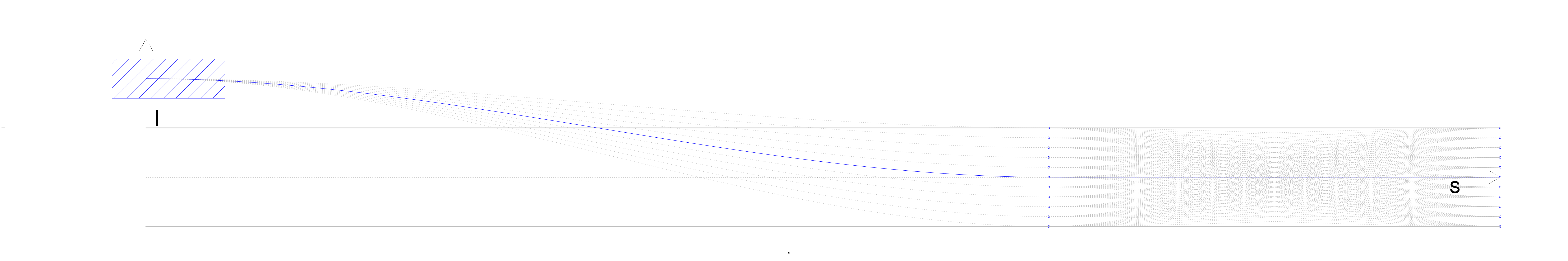}
\captionsetup[figure]{labelfont={sc},textfont=normalfont,singlelinecheck=on,justification=centered,labelsep=colon}
\caption{Dynamic programming path optimizer sampling for default lane and change lane}
\label{fig:dp_path}
\end{center}
\end{figure*}

\subsection{M-Step Spline QP Path}
The spline QP path step is a refinement of the dynamic programming path step. In a dynamic programming path, a feasible tunnel is generated based on the selected path. Then, the spline-based QP step will generate a smooth path within this feasible tunnel, as shown in Fig.~\ref{fig:spline_qp_path_1}. 

\begin{figure}[htbp]
\begin{center}
\includegraphics[width = 0.45\textwidth]{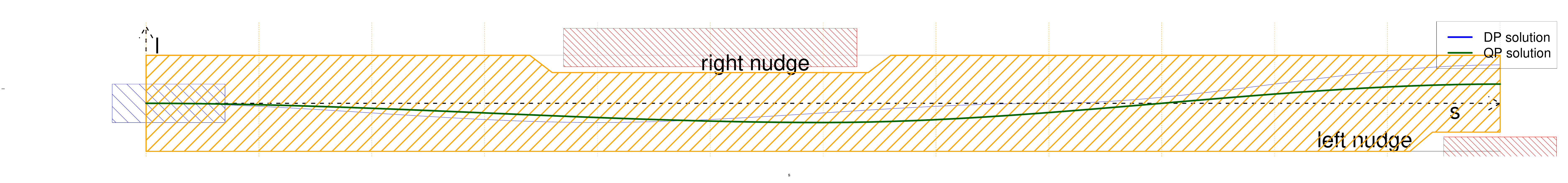}
\captionsetup[figure]{labelfont={sc},textfont=normalfont,singlelinecheck=on,justification=centered,labelsep=colon}
\caption{Spline QP Path Example}
\label{fig:spline_qp_path_1}
\end{center}
\end{figure}

The spline QP path is generated by optimizing an objective function with a linearized constraint through the QP spline solver. Fig.~\ref{fig:spline_qp_path_structure} shows the pipeline of the QP path step.

\begin{figure}[htbp]
\begin{center}
\includegraphics[width = 0.45\textwidth]{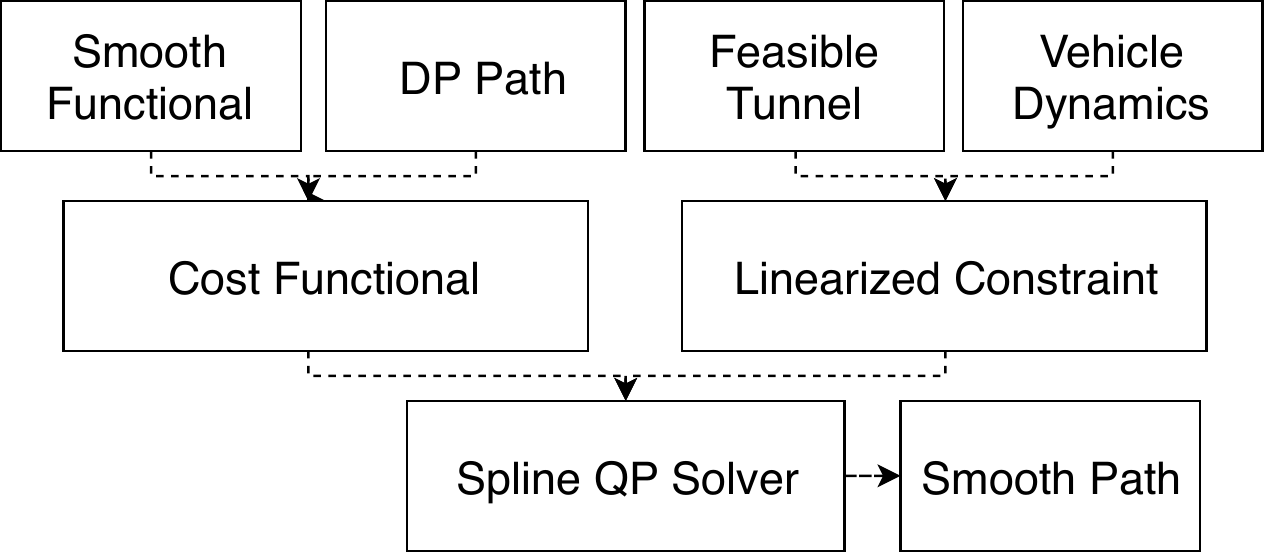}
\captionsetup[figure]{labelfont={sc},textfont=normalfont,singlelinecheck=on,justification=centered,labelsep=colon}
\caption{Spline QP Path Example}
\label{fig:spline_qp_path_structure}
\end{center}
\end{figure}

The objective function of the QP path is a linear combination of smoothness costs and guidance line cost. The guidance line in this step is the DP path. The guidance line provides an estimate of the obstacle nudging distance. Mathematically, the QP path step optimizes the following functional:
\begin{align*}
C_{s}(f) &= w_1 \int (f'(s))^2 ds + w_2 \int (f''(s))^2 ds \\
&+ w_3 \int (f'''(s))^2 + w_4 \int(f(s) - g(s))^2ds.
\end{align*}
where $g(s) $ is the DP path result. $f'(s)$, $f''(s)$ and $f'''(s)$ are related to the heading, curvature and derivative of curvature. The objective function describes the balance between nudging obstacles and smoothness. 

The constraints in the QP path include boundary constraints and  dynamic feasibility. These constraints are applied on $f(s), f'(s)$ and $f''(s)$ at a sequence of station coordinates $s_0, s_1, ...., s_n$. To extract boundary constraints, the feasible ranges at station points are extracted. The feasible range at each point is described as $(l_{low, i}, \leq l_{high, i})$. In EM planner, the ego vehicle is considered under the bicycle model. Thus, simply providing a range for $l = f(s)$ is not sufficient since the heading of the ego car also matters.

As shown in Fig.~\ref{fig:spline_qp_path_2}, to keep the boundary constraint convex and linear, we add two half circles on the front and rear ends of the ego car. Denote the front-to-rear wheel center distance as $l_f$ and the vehicle width as $w$. Then, the lateral position of the left-front corner is given by
$$l_{\text{left front corner}} =  f(s) + sin(\theta) l_f + w / 2 $$, 
where $\theta$ is the heading difference between the ego car and road station direction. The constraint can be further linearized using the following inequality approximation:
\begin{align*}
f(s) + sin(\theta) l_f + w / 2 &\leq  f(s) + f'(s) l_r + w / 2\\
& \leq l_{\text{left corner bound}}
\end{align*}
Similarly, the linearization can be applied on the remaining three corners. The linearized constraints are good enough since $\theta$ is generally small. For $\theta < pi / 12 $, the estimation will be less than 2 - 3 cm conservative on the lateral direction compared to the constraint without linearization.

\begin{figure}[htbp]
\begin{center}
\includegraphics[width = 0.45\textwidth]{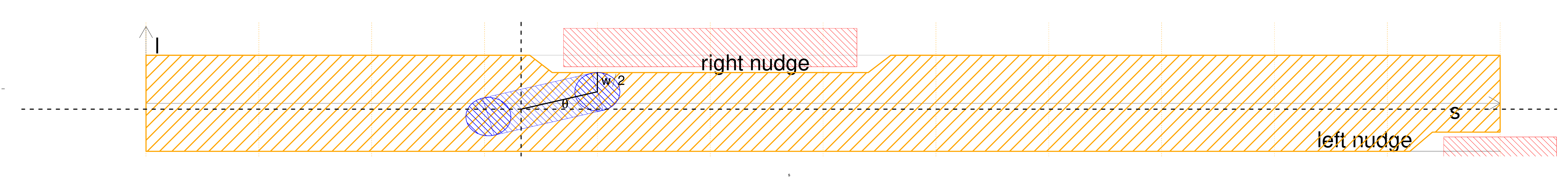}
\captionsetup[figure]{labelfont={sc},textfont=normalfont,singlelinecheck=on,justification=centered,labelsep=colon}
\caption{Spline QP Path Constraint Linearization}
\label{fig:spline_qp_path_2}
\end{center}
\end{figure}
The range constraints of $f''(s)$ and $f'''(s)$ can also be used as dynamic feasibility since they are related to curvature and the curvature derivative. In addition to the boundary constraint, the generated path shall match the ego car's initial lateral position and derivatives $(f(s_0), f'(s_0), f''(s_0))$. Since all constraints are linear with respect to spline parameters, a quadratic programming solver can be used to solve the problem very fast.

The details of the smoothing spline and quadratic programming problem are covered in \ref{sec:appendix}.

\subsection{M-Step DP Speed Optimizer}
The speed optimizer generates a speed profile in the ST graph, which is represented as a station function with respect to time $S(t)$. Similar to in the path optimizer, finding a best speed profile on the ST graph is a non-convex optimization problem. We use dynamic programming combined with spline quadratic programming to find a smooth speed profile on the ST graph.
In Fig.~\ref{fig:dp_speed_structure}, the DP speed step includes a cost functional, ST graph grids and dynamic programming search. The generated result includes a piecewise linear speed profile, a feasible tunnel and obstacle speed decisions, as shown in Fig.~\ref{fig:dp_speed_1}. The speed profile will be used in the spline QP speed step as a guidance line, and the feasible tunnel will be used to generate a convex region.
\begin{figure}[htbp]
\begin{center}
\includegraphics[width = 0.4\textwidth]{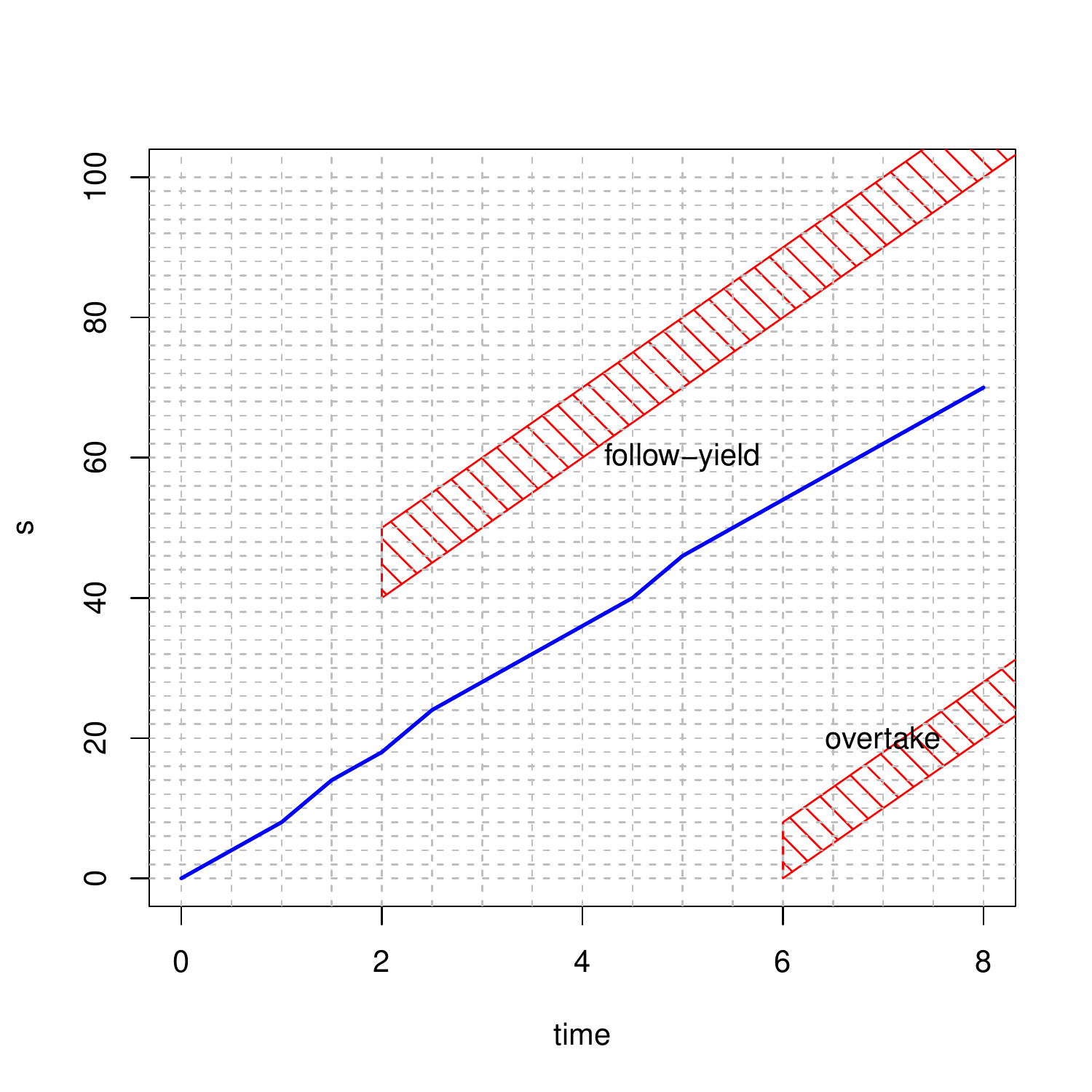}
\captionsetup[figure]{labelfont={sc},textfont=normalfont,singlelinecheck=on,justification=centered,labelsep=colon}
\caption{DP Speed Optimizer}
\label{fig:dp_speed_1}
\end{center}
\end{figure}

In detail, obstacle information is first discretized into grids on the ST graph. Denote $(t_0, t_1, ..., t_n)$ as equally spaced evaluated points on the time axis with interval $dt$. A piecewise linear speed profile function is represented as $S = (s_0, s_1, ..., s_n)$ on the grids. Furthermore, the derivatives are approximated by the finite difference method. 
\begin{eqnarray*} s_i' = v_i \approx& \frac{s_i - s_{i - 1}}{dt} \\
s_i'' = a_i \approx& \frac{s_i - 2 s_{i - 1} + s_{i - 2}}{(dt)^2} \\
s_i''' = j_i  \approx& \frac{s_i - 3 s_{i - 1}  - 3 s_{i - 2} + s_{i - 3}}{(dt)^3}\\
\end{eqnarray*}
The goal is to optimize a cost functional in the ST graph within the constraints. In detail, the cost for the DP speed optimizer is represented as follows:
\begin{align*}
 C_{total}(S) &= w_1 \int_{t_0}^{t_n} g(S' - V_{ref}) dt \\
 &+w_2 \int_{t_0}^{t_n} (S'')^2 dt + w_3\int_{t_0}^{t_n} (S''')^2 dt \\
 &+ w_4 C_{obs}(S)
 \end{align*}
The first term is the velocity keeping cost. This term indicates that the vehicle shall follow the designated speed when there are no obstacles or traffic light restrictions present.  $V_{ref}$ describes the reference speed, which is determined by the road speed limits, curvature and other traffic regulations. The $g$ function is designed to have different penalties for values that are less or greater than $V_{ref}$. The acceleration and jerk square integral describes the smoothness of the speed profile. The last term, $C_{obs}$, describes the total obstacle cost. The distances of the ego car to all obstacles are evaluated to determine the total obstacle costs.

The dynamic programming search space is also within the vehicle dynamic constraints. The dynamic constraints include acceleration, jerk limits and a monotonicity constraint since we require that the generated trajectories do not perform backing maneuvers when driving on the road. Backing can only be performed under parking or other specified scenarios. The search algorithm is straightforward; some necessary pruning based on vehicle dynamic constraints is also applied to accelerate the process. We will not discuss the search part in detail.

\begin{figure}[htbp]
\begin{center}
\includegraphics[width = 0.45\textwidth]{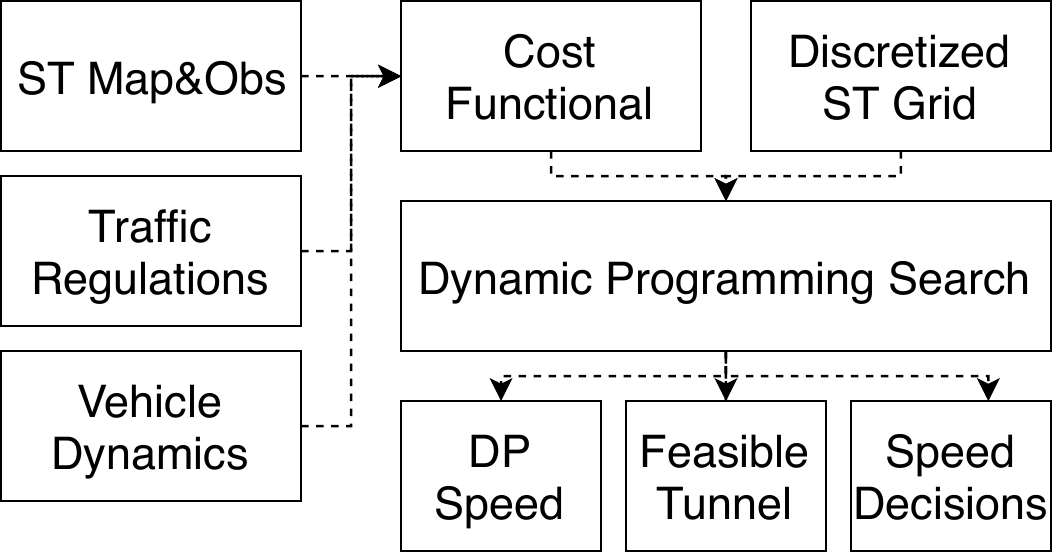}
\captionsetup[figure]{labelfont={sc},textfont=normalfont,singlelinecheck=on,justification=centered,labelsep=colon}
\caption{DP Speed Optimizer}
\label{fig:dp_speed_structure}
\end{center}
\end{figure}

\subsection{M-Step QP Speed Optimizer}
Since the piecewise linear speed profile cannot satisfy dynamic requirements, the spline QP step is needed to fill this gap. In Fig.~\ref{fig:qp_speed_structure}, the spline QP speed step includes three parts: cost functional, linearized constraint and spline QP solver. 

The cost functional is described as follows:
\begin{align*}
C_{total}(S) &= w_1 \int_{t_0}^{t_n} (S - S_{ref})^2 dt + w_2 \int_{t_0}^{t_n} (S'')^2 dt\\
& + w_3\int_{t_0}^{t_n} (S''')^2 dt.
\end{align*}
The first term measures the distance between the DP speed guidance profile $S_{ref}$ and generated path $S$. The acceleration and jerk terms are measures of the speed profile smoothness. Thus, the objective function is a balance between following the guidance line and smoothness.
\begin{figure}[htbp]
\begin{center}
\includegraphics[width = 0.4\textwidth]{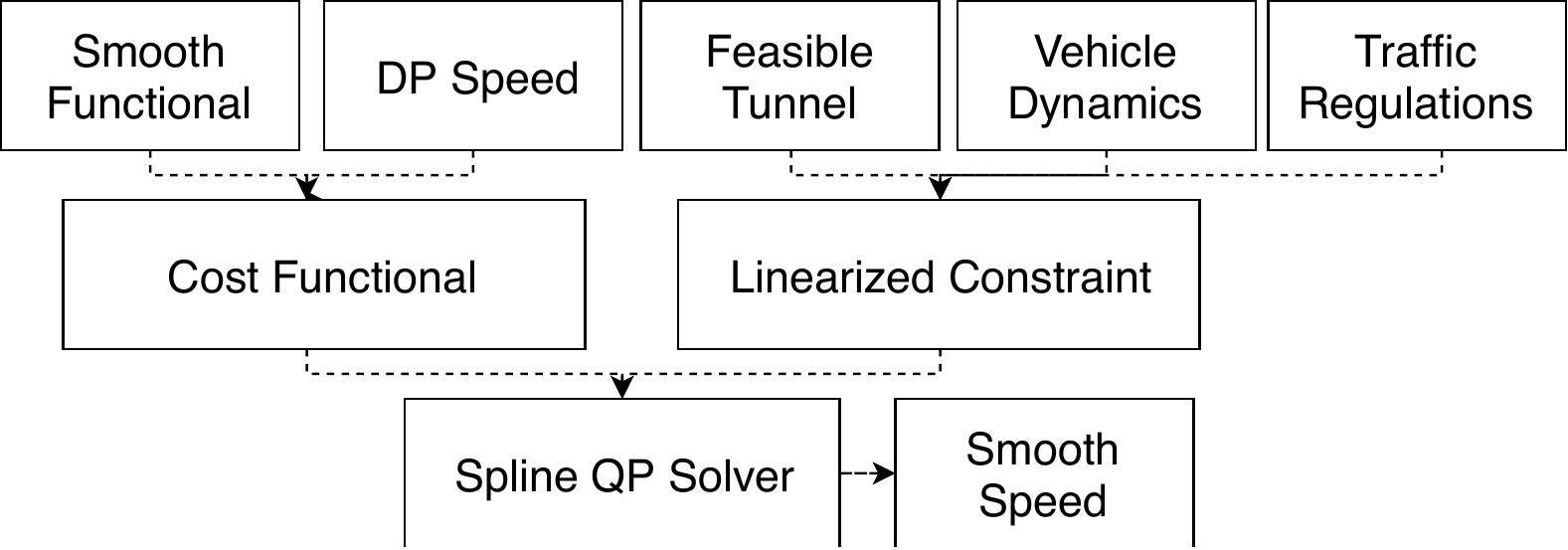}
\captionsetup[figure]{labelfont={sc},textfont=normalfont,singlelinecheck=on,justification=centered,labelsep=colon}
\caption{Spline QP speed optimizer}
\label{fig:qp_speed_structure}

\end{center}
\end{figure}

The spline optimization is within the linearized constraint. The constraints in QP speed optimization include the following boundary constraints:
\[S(t_i) \leq S(t_{i + 1}), i = 0, 1, 2, ..., n - 1,\]
\[S_{l, t_i} \leq S(t_i) \leq S_{u, t_i},\]
\[ S'(t_i) \leq V_{upper},\]
\[ -Dec_{max} \leq S''(t_i) \leq Acc_{max}\]
\[ -J_{max} \leq S'''(t_i) \leq J_{max}\]
as well as constraints for matching the initial velocity and acceleration. The first constraint is monotonicity evaluated at designated points. The second, third, and fourth constraints are requirements from traffic regulations and vehicle dynamic constraints. After wrapping up the cost objective and constraints, the spline solver will generate a smooth feasible speed profile as in Fig.~\ref{fig:qp_speed_1}. Combined with the path profile, EM planner will generate a smooth trajectory for the control module.

\begin{figure}[htbp]
\begin{center}
\includegraphics[width = 0.5\textwidth]{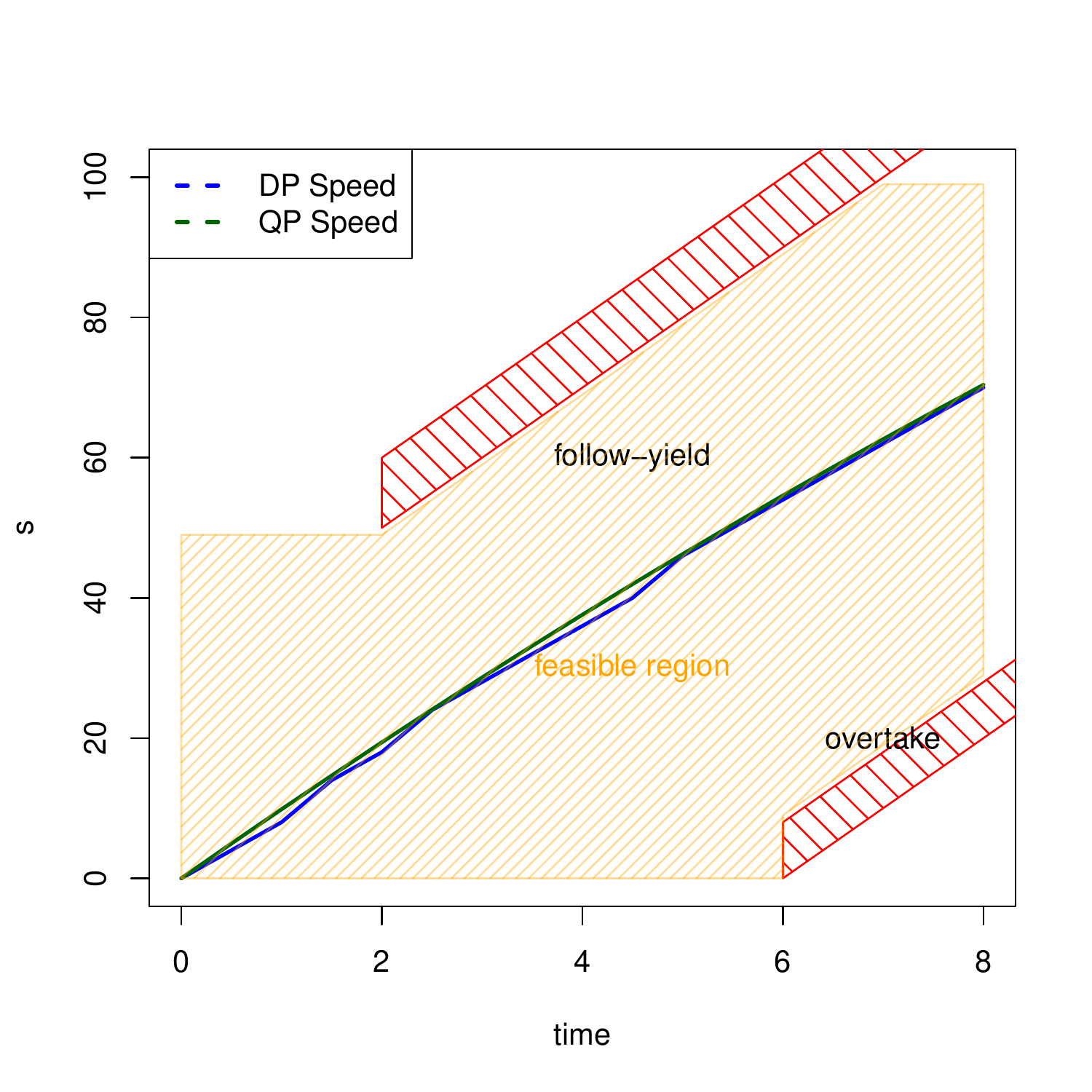}
\captionsetup[figure]{labelfont={sc},textfont=normalfont,singlelinecheck=on,justification=centered,labelsep=colon}
\caption{Spline QP speed optimizer}
\label{fig:qp_speed_1}
\end{center}
\end{figure}

\subsection{Notes on Solving Quadratic Programming Problems}
For safety considerations, we evaluate the path and speed at approximately one-hundred different locations or time points. The number of constraints is greater than six hundred. For both the path and speed optimizers, we find that piecewise quintic polynomials are good enough. The spline generally contains 3 to 5 polynomials with approximately 30 parameters. Thus, the quadratic programming problem has a relatively small objective function but  large number of constraints. Consequently, an active set QP solver is good for solving the problem. In addition to accelerating the quadratic programming, we use the result calculated in the last cycle  as a hot start. The QP problem can be solved within 3 ms on average, which satisfies our time consumption requirement.

\subsection{Notes on Non-convex Optimization With DP and QP}
DP and QP alone both have their limitations in the non-convex domain. A combination of DP and QP will take advantage of the two and reach an ideal solution.

\begin{itemize}
\item{
\textit{DP:}
As described earlier in this manuscript, the DP algorithm depends on a sampling step to generate candidate solutions. 
Because of the restriction of processing time, the number of sampled candidates is limited by the sampling grid.
Optimization within a finite grid yields a rough DP solution.
In other words, DP does not necessarily, and in almost all cases would not, deliver the optimal solution. 
For example, DP could select a path that nudges the obstacle from the left but not nudge with the best distance.
}

\item{
\textit{QP:}
Conversely, QP generates a solution based on the convex domain. It is not available without the help of the DP step.
For example, if an obstacle is in front of the master vehicle, QP requires a decision, such as nudge from the left, nudge from the right, follow, or overtake, to generate its constraint.
A random or rule-based decision will make QP susceptible to failure or fall into a local minimal. 
}

\item{
\textit{DP + QP:}
A DP plus QP algorithm will minimize the limitations of both:
(1) The EM planner first uses DP to search within a grid to reach a rough resolution.
(2) The DP results are used to generate a convex domain and guide QP.
(3) QP is used to search for the optimal solution in the convex region that most likely contains global optima. 
}
\end{itemize}

\section{Case Study}\label{sec:example}
As mentioned in the above sections, although most state-of-the-art planning algorithms are based on heavy decisions, EM planner is a light-decision-based planner. It is true that a heavy-decision-based algorithm, or heavily rule-based algorithm, is easily understood and explained. The disadvantages are also clear: it may be trapped in corner cases (while its frequency is closely related to the complexity and magnitude of the number of rules) and not always be optimal.
In this section, we will illustrate the benefits of the light-decision-based planning algorithm by presenting several case studies. These cases were exposed during the intense daily test routines in Baidu's heavy-decision planning modules and solved by the latest light-decision planning module.

\begin{figure*}[htbp]
\begin{center}
    \begin{subfigure}[b]{0.7\textwidth}
    	\vspace{-0.5cm}
	\includegraphics[width = \textwidth]{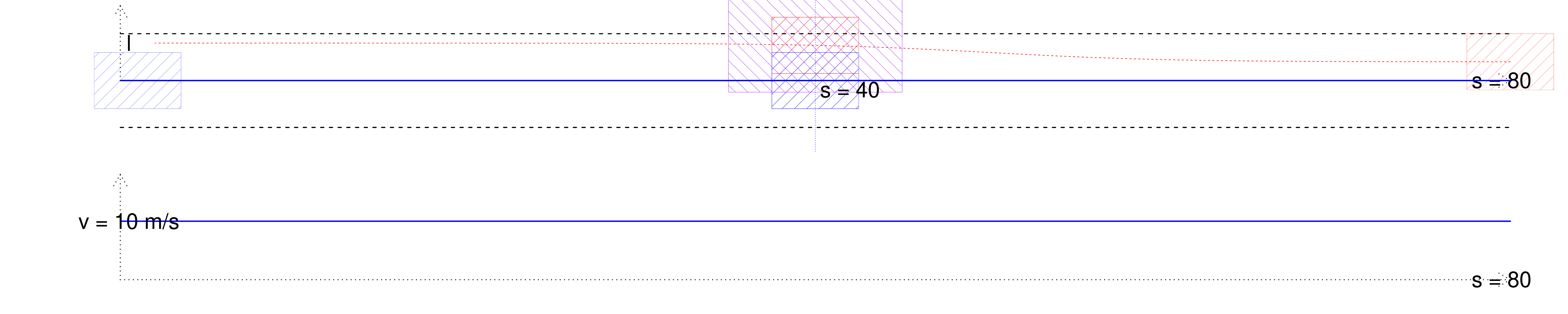}
	\vspace{-0.7cm}
        \caption{Stage A: Historical Planning}
        \label{fig:cs1}
        \vspace{0.3cm}
    \end{subfigure}
    \begin{subfigure}[b]{0.7\textwidth}
	\includegraphics[width = \textwidth]{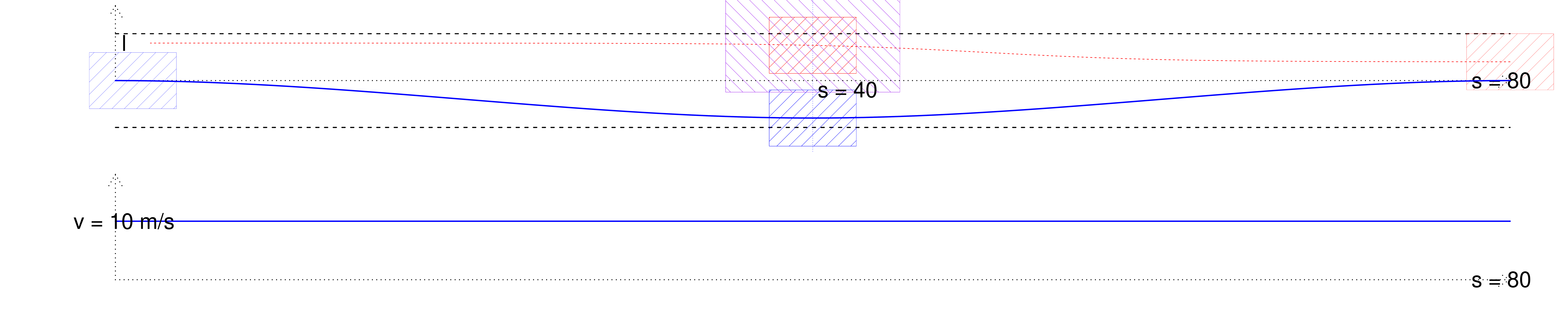}
	\vspace{-0.7cm}
        \caption{Stage B: Path Planning Cycle 1}
         \label{fig:cs2}
        \vspace{0.3cm}
    \end{subfigure}
        \begin{subfigure}[b]{0.7\textwidth}
	\includegraphics[width = \textwidth]{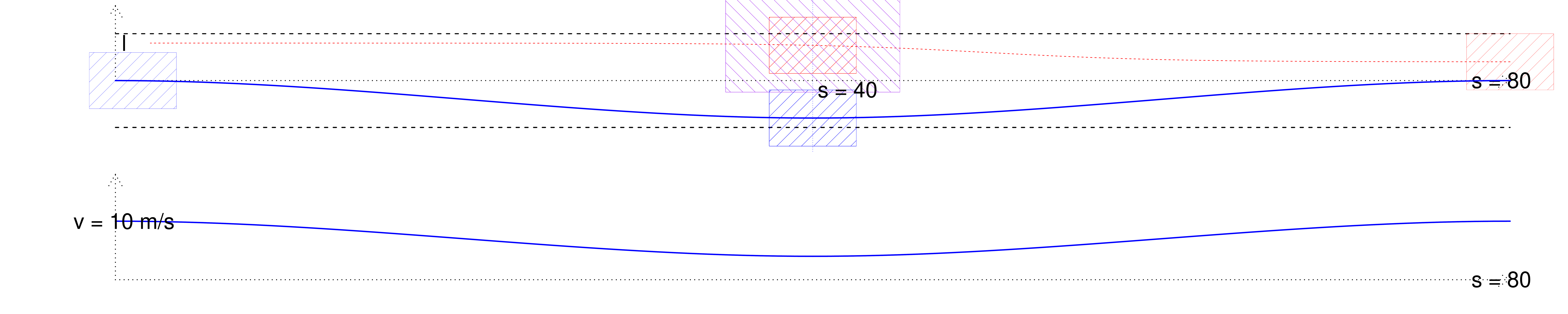}
	\vspace{-0.7cm}
        \caption{Stage C: Speed Planning Cycle 1}
         \label{fig:cs3}
        \vspace{0.3cm}
    \end{subfigure}
      \begin{subfigure}[b]{0.7\textwidth}
	\includegraphics[width = \textwidth]{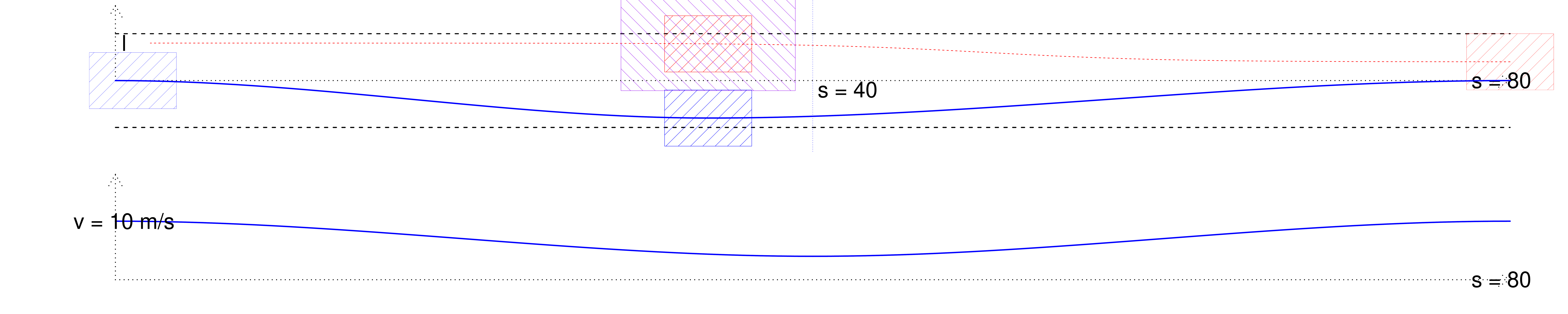}
	\vspace{-0.7cm}
        \caption{Stage D: Path Planning Cycle 2}
         \label{fig:cs4}
        \vspace{0.3cm}
    \end{subfigure}
          \begin{subfigure}[b]{0.7\textwidth}
	\includegraphics[width = \textwidth]{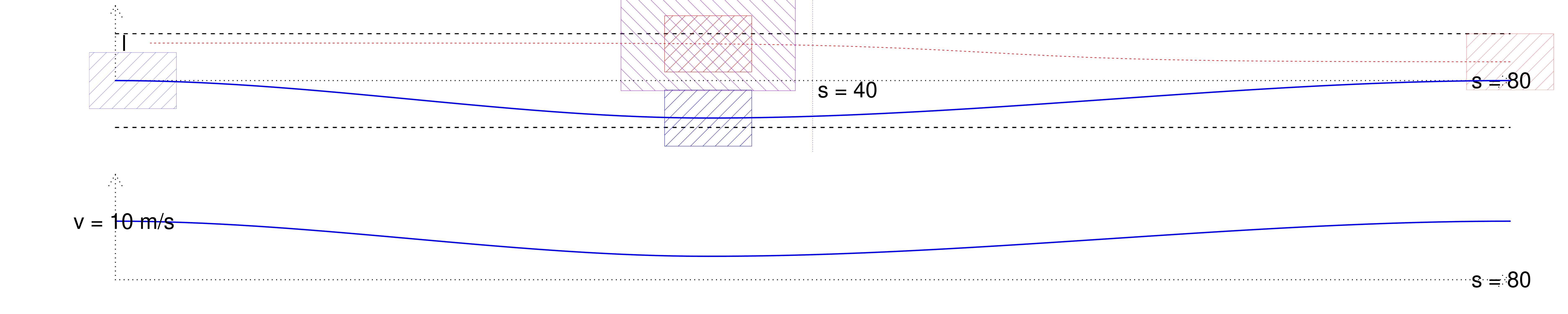}
	\vspace{-0.7cm}
        \caption{Stage E: Speed Planning Cycle 2}
         \label{fig:cs5}
    \end{subfigure}

\caption{Case Study - Nudge oncoming dynamic obstacle}
This figure shows how the EM planner manages to iteratively solve the update path and speed profile.
\label{fig:casestudy1}
\end{center}
\end{figure*}

Figure \ref{fig:casestudy1} is a hands-on example of how EM planner iterates within and between planning cycles to achieve the optimal trajectory. 
In this case study, we demonstrate how the trajectory is generated while an obstacle enters our path.
Assuming that the master vehicle has a speed of 10 meters per second and there is a dynamic obstacle that is moving toward us in the opposite direction with a speed that is also 10 meters per second, EM planner generates the path and speed profile iteratively with the steps below.

\begin{enumerate}

\item{
\textit{Historical Planning (Figure \ref{fig:cs1}).} 
In the historical planning profile, i.e., before the dynamic obstacle enters, the master vehicle is moving forward straight with a constant speed of 10 meters per second.
}

\item{
\textit{Path Profile Iteration 1 (Figure \ref{fig:cs2}).} 
During this step, the speed profile is cruising at 10 m/s from the historical profile. 
Based on this cursing speed, the master vehicle and the dynamic obstacle will meet each other at position $S=40m$.
Consequently, the best way to avoid this obstacle is to nudge it from the right side at $S=40m$.
}

\item{
\textit{Speed Profile Iteration 1 (Figure \ref{fig:cs3}).}
Based on the path profile, which is nudge from the right, from step 1,
the master vehicle adjusts its speed according to its interaction with the obstacle.
Thus, the master vehicle will slow to 5 m/s when passing an obstacle with a slower speed, as passengers may expect.
}

\item{
\textit{Path Profile Iteration 2 (Figure \ref{fig:cs4}).}
Under the new speed profile, which is slower than the original one, the master vehicle no longer passes the dynamic obstacle at $S=40m$ but rather a new position at $S=30m$. 
Thus, the path to nudge the obstacle should be updated to a new one to maximize the nudge distance at $S=30m$.
}

\item{
\textit{Speed Profile Iteration 2 (Figure \ref{fig:cs5}).}
Under the new path profile, where the nudge is performed at $S=30m$, the slow down at $S=40m$ is no longer necessary. The new speed profile indicates that the master vehicle can accelerate at $S=40m$ and still generate a smooth pass at $S=30m$.
}

\end{enumerate}
Thus, the final trajectory based on the four steps is to slow to nudge the obstacle at S=30 m and then accelerate after the master vehicle passes the obstacle, which is very likely how human drivers perform under this scenario.

Note that it is not necessary to always take exactly four steps to create the plan. It could take fewer or more steps depending on the scenario.
In general, the more complicated the environment is, the more steps that may be required.

\section{Computational Performance}\label{sec:performance}

Because the three-dimensional station-lateral-speed problem has been split into two two-dimensional problems, i.e., station-lateral problem and station-speed problem, the computational complexity of EM planner has been significantly decreased, and thus, this planner is very efficient. 
Assuming that we have n obstacles with M candidate path profiles and N candidate speed profiles, the computational complexity of this algorithm is O(\textit{n(M+N)}).

On a PIC of Nuvo-6108GC-GTX1080-E3-1275, with DDR4-16GB-ECC and HDD1TB-72 
\cite{http://apollo.auto/platform/hardware.html}, it takes less than 100 ms on average.

\section{Conclusion}\label{sec:conclusion}

EM planner is a light-decision-based algorithm. 
Compared with other heavy-decision-based algorithms, the advantage of EM planner is its ability to perform under complicated scenarios with multiple obstacles. When heavy-decision-based methods attempt to predetermine how to act with each obstacle, the difficulties are significant: 
(1) It is difficult to understand and predict how obstacles interact with each other and the master vehicle; thus, their following movement is hard to describe and therefore hard to be considered by any rules.
(2) With multiple obstacles blocking the road, the probability of not finding a trajectory that meets all predetermined decisions is dramatically reduced, leading  to planning failure.

One critical issue in autonomous driving vehicles is the challenge of safety vs. passability. 
A strict rule increases the safety of the vehicle but lowers the passability, and vice versa.
Take the lane-changing case as an example; one could easily pause the lane-changing process if there is a vehicle behind with simple rules. This could grant safety but considerably decreases the passability.
EM planner described in this manuscript is also designed to solve the inconsistency of potential decisions and planning, while it also improves the passability of autonomous driving vehicles.

EM planner significantly reduces computational complexity by transforming a three-dimensional station-lateral-speed problem into two two-dimensional station-lateral/station-speed problems. It could significantly reduce the processing time and therefore increase the interaction ability of the whole system.

As of May 16th, 2018, the effectivity of this system has been proven under 3,380 hours and approximately 68,000 kilometers (42,253 miles) of intense closed-loop testing in Baidu Apollo autonomous driving vehicles. 
The algorithm has been evaluated under different countries, traffic laws and conditions, including extremely crowded urban scenarios such as Beijing, China, and Sunnyvale, CA, USA.
The algorithm has also been evaluated and tested in more than one-hundred-thousand hours and a million kilometers (0.621 million miles) simulation test.

The algorithm described in this manuscript is available at 
\url{https://github.com/ApolloAuto/apollo-11/tree/master/modules/planning}.

\section*{Acknowledgement}
We would like to thank Apollo Community for their useful suggestions and contributions.

\clearpage
\section*{Appendix 1}

Apollo is an open autonomous driving platform. It is a high-performance flexible architecture that supports fully autonomous driving capabilities. For business contact, please visit \url{http://apollo.auto}.

A list of contributors for Apollo includes, but is not limited to, \url{https://github.com/ApolloAuto/apollo-11/graphs/contributors}.

\section*{Appendix 2: constrained smoothing spline and quadratic programming}\label{sec:appendix}
The spline QP path and speed optimizer uses the same constrained smoothing spline framework. In this section, we formalize the smoothing spline problem within the quadratic programming framework and linearized constraints. Optimization with a quadratic convex objective and linearized constraints can be rapidly and stably solved.

Under the quadratic programming setting, we optimize a third-order smooth function $f(x)$ given a quadratic objective functional and linear constraints. A smoothing spline function $f(x)$ on domain $(a, b), a = x_0, b = x_n$ is represented as a linear combination of piecewise polynomial bases. Denote $x_0 = a, x_1, x_2, ..., x_n = b$ as the knots. Then, $f(x)$ is given by the following:
\begin{equation} \label{eqn:functionaloptimization} f(x) = \begin{cases}  f_0(x - x_0)& x \in [x_0, x_1) \\
		                    f_1(x - x_1) & x  \in [x_1, x_2) \\
		                     & ... \\
		                    f_{n - 1} (x - x_{n - 1}) & x \in [x_{n - 1}, x_{n}] 
		                          \end{cases} \end{equation}
where $f_{k}(x) = p_{k0} + p_{k1} x + ... p_{km} x^{m}$ is a polynomial function. The coefficient is defined as $\mathbf p_k = (p_{k0}, p_{k1} ,..., p_{km})^T$ for $f_k$. The smoothing spline parameter vector is defined as $\mathbf p = (\mathbf p_0^T, \mathbf p_1^T, ..., \mathbf p_{n - 1}^T)^T$.
Mathematically, the spline solver attempts to find a function $\hat f$ that 
optimizes the linear combination of functionals on the reproducing kernel Hilbert space (RKHS) domain. The domain space is described as
 
 \begin{align*}
 &\Omega = \{f: [a, b] \rightarrow \mathbf R | f, f^{(1)}, f^{(2)}, f^{(3)}  \\
 & \text{is abs. conti. and} \int_{a}^b (f^{(m)})^2 dx < \infty, m = 0, 1, 2, 3 \} \\
 \end{align*}

The spline QP problem is formed with an objective functional, linearized constraints and quadratic programming solver. It is described as follows:
\begin{equation}
\arg \min_{f \in \Omega} \mathbf P(f) = \sum_{i = 0}^{3}w_i \mathbf P_i (f)
\label{eqn:spline}
\end{equation}
with $\mathbf L(f(x)) <= \mathbf 0$.
The objective functional $\mathbf P(f)$ is a linear combination of four functionals, $\mathbf P_i, i = 0, 1, 2, 3$ on $\Omega$. Specifically,
\begin{align*}
& \mathbf  P_0 (f)  =  \int_{a}^b  (f(x) - g(x))^2 dx,\\
& \mathbf  P_1 (f)  =  \int_{a}^b  (f^{(1)}(x))^2 dx, \\
& \mathbf  P_2 (f)  =  \int_{a}^b  (f^{(2)}(x))^2 dx,\\
& \mathbf  P_3 (f)  =  \int_{a}^b  (f^{(3)}(x))^2 dx, \\
\end{align*}
where g is a pre-specified guideline function in $\Omega$.
The first one represents the distance to the guidance line $g$, whereas the remaining three describe the smoothness of $f$. The constraint functional $\mathbf L(f(x))$ includes the following types based on the problem setup:
\begin{enumerate}
\item  Boundary constraint, i.e., $ L \leq f^{(m)}(x) \leq U, m = 0, 1, 2, 3$ given evaluated location $x$. For example, $m = 0, x = 2$ represents the smooth function value; then, $f(2)$ shall be bounded between lower bound $l$ and upper bound $L$. Specifically, boundary constraint with heading has the form $L \leq f(x) + c f'(x) \leq U $.
\item Smoothness constraint up to the $m$-th derivative. Since $f$ is piecewise polynomials, it requires polynomials are joint at spline knots $x_k, k = 1, 2, ..., n-1$ with up to $m$-th-order derivative matching.
\item Monotonicity constraint. The constraint is specifically designed for the speed smoother, which guarantees that the path is monotonic at a sequence of specified points.
\end{enumerate}

\begin{theorem}
Functional $\mathbf P$ defined in \ref{eqn:spline} has a quadratic form, and functional $L$ has a linear form with respect to parameter $\mathbf p$.
\end{theorem}
\begin{proof}
For simplicity, denote variable vector $\mathbf x_{(i)}, i = 0, 1, 2, 3$ as the $i$-th-order derivative coefficient of $f$ with respect to parameter $\mathbf p$ at $x \in [x_k, x_{k + 1}), k = 0, 1, ..., n - 1$. That is,

\begin{align*}
\mathbf x_{(0)}  =& (1, \tilde x, \tilde x^{2}, ..., \tilde x^{m})^T, \\
\mathbf x_{(1)}  =& (0, 1, 2 \tilde x, ..., \tilde m x^{m - 1})^T, \\
\mathbf x_{(2)}  =& (0, 0, 2, ..., m (m - 1) \tilde x^{m - 2})^T, \\
\mathbf x_{(3)}  =& (0, 0, 0, ..., m(m - 1)(m - 2)\tilde x^{m - 3})^T,\\
\end{align*}
where $\tilde x = x - x_k, ~ x \in [x_k, x_{k + 1}), k = 0, 1, ..., n - 1.$
In (\ref{eqn:functionaloptimization}), each $\mathbf P_i, i = 1, 2, 3$ can be represented as a summation of integrations:
\begin{align*}
&\mathbf P_i(f) = \sum_{k = 0}^{n - 1} \int_{x_{k}}^{x_{k + 1}} (f^{(i)})^2 dx \\
&= \sum_{k = 0}^{n - 1} \int_{x_{k}}^{x_{k + 1}} \mathbf p_k^T \mathbf x_{(i)}\mathbf x^T_{(i)} \mathbf p_k dx \\
&= \mathbf p^T \mathbf X_{(i)} \mathbf p
\end{align*}
where $X_{(i)}$ is defined as the block diagonal matrix $Diag (\int_{x_0}^{x_{1}} \mathbf x_{(i)} \mathbf x^T_{(i)} dx, ..., \int_{x_k}^{x_{k + 1}} \mathbf x_{(i)} \mathbf x^T_{(i)} dx) $.
Similarly, for $\mathbf P_0(f)$,
\begin{align*}
&\mathbf P_0(f)  = \sum_{k = 0}^{n - 1} \int_{x_{k}}^{x_{k + 1}} (f(x) - g(x))^2 dx = \\
&\sum_{k = 0}^{n - 1} \int_{x_{k}}^{x_{k + 1}} (\mathbf p_k^T \mathbf x_{(0)}\mathbf x^T_{(i)} \mathbf p_k - 2 \mathbf p_k^T \mathbf x_{(0)} g(x) + (g(x))^2) dx  \\
& = \mathbf p^T \mathbf X_{(0)} \mathbf p - 2 \mathbf p^T \mathbf c + const.
\end{align*}
where  $X_{(0)}$ is defined as the diagonal block matrix $Diag (\int_{x_0}^{x_{1}} \mathbf x_{(0)} \mathbf x^T_{(0)} dx, ..., \int_{x_k}^{x_{k + 1}} \mathbf x_{(0)} \mathbf x^T_{(0)} dx) $ and $\mathbf c = \int_{x_k}^{x_{k + 1}} \mathbf x_{(0)} g(x) dx$.

Thus, the objective function $\mathbf P(f) $ has a quadratic convex form with respect to spline parameter vector $\mathbf p$.
\end{proof}

\bibliographystyle{IEEEtran}
\bibliography{IEEEabrv,reference}

\begin{thebibliography}{10}
\providecommand{\url}[1]{#1}
\csname url@samestyle\endcsname
\providecommand{\newblock}{\relax}
\providecommand{\bibinfo}[2]{#2}
\providecommand{\BIBentrySTDinterwordspacing}{\spaceskip=0pt\relax}
\providecommand{\BIBentryALTinterwordstretchfactor}{4}
\providecommand{\BIBentryALTinterwordspacing}{\spaceskip=\fontdimen2\font plus
\BIBentryALTinterwordstretchfactor\fontdimen3\font minus
  \fontdimen4\font\relax}
\providecommand{\BIBforeignlanguage}[2]{{%
\expandafter\ifx\csname l@#1\endcsname\relax
\typeout{** WARNING: IEEEtran.bst: No hyphenation pattern has been}%
\typeout{** loaded for the language `#1'. Using the pattern for}%
\typeout{** the default language instead.}%
\else
\language=\csname l@#1\endcsname
\fi
#2}}
\providecommand{\BIBdecl}{\relax}
\BIBdecl

\bibitem{dempster1977maximum}
A.~P. Dempster, N.~M. Laird, and D.~B. Rubin, ``Maximum likelihood from
  incomplete data via the em algorithm,'' \emph{J. Royal Stat. Soc. Ser. B
  (Methodol.)}, vol.~39, pp. 1--38, 1977.

\bibitem{ajanovic2018search}
Z.~Ajanovic, B.~Lacevic, B.~Shyrokau, M.~Stolz, and M.~Horn, ``Search-based
  optimal motion planning for automated driving,'' \emph{arXiv preprint
  arXiv:1803.04868}, 2018.

\bibitem{werling2010optimal}
M.~Werling, J.~Ziegler, S.~Kammel, and S.~Thrun, ``Optimal trajectory
  generation for dynamic street scenarios in a frenet frame,'' in \emph{2010
  IEEE International Conference on Robotics and Automation (ICRA)}.\hskip 1em
  plus 0.5em minus 0.4em\relax IEEE, 2010, pp. 987--993.

\bibitem{mcnaughton2011motion}
M.~McNaughton, C.~Urmson, J.~M. Dolan, and J.-W. Lee, ``Motion planning for
  autonomous driving with a conformal spatiotemporal lattice,'' in \emph{2011
  IEEE International Conference on Robotics and Automation (ICRA)}.\hskip 1em
  plus 0.5em minus 0.4em\relax IEEE, 2011, pp. 4889--4895.

\bibitem{ziegler2009spatiotemporal}
J.~Ziegler and C.~Stiller, ``Spatiotemporal state lattices for fast trajectory
  planning in dynamic on-road driving scenarios,'' in \emph{IEEE/RSJ
  International Conference on Intelligent Robots and Systems, 2009 (IROS
  2009)}.\hskip 1em plus 0.5em minus 0.4em\relax IEEE, 2009, pp. 1879--1884.

\bibitem{gu2015tunable}
T.~Gu, J.~Atwood, C.~Dong, J.~M. Dolan, and J.-W. Lee, ``Tunable and stable
  real-time trajectory planning for urban autonomous driving,'' in \emph{2015
  IEEE/RSJ International Conference on Intelligent Robots and Systems
  (IROS)}.\hskip 1em plus 0.5em minus 0.4em\relax IEEE, 2015, pp. 250--256.

\bibitem{xu2012real}
W.~Xu, J.~Wei, J.~M. Dolan, H.~Zhao, and H.~Zha, ``A real-time motion planner
  with trajectory optimization for autonomous vehicles,'' in \emph{2012 IEEE
  International Conference on Robotics and Automation (ICRA)}.\hskip 1em plus
  0.5em minus 0.4em\relax IEEE, 2012, pp. 2061--2067.

\bibitem{montemerlo2008junior}
M.~Montemerlo, J.~Becker, S.~Bhat, H.~Dahlkamp, D.~Dolgov, S.~Ettinger,
  D.~Haehnel, T.~Hilden, G.~Hoffmann, B.~Huhnke \emph{et~al.}, ``Junior: The
  stanford entry in the urban challenge,'' \emph{J. Field Robot.}, vol.~25,
  no.~9, pp. 569--597, 2008.

\bibitem{urmson2008autonomous}
C.~Urmson, J.~Anhalt, D.~Bagnell, C.~Baker, R.~Bittner, M.~N. Clark, J.~Dolan,
  D.~Duggins, T.~Galatali, C.~Geyer \emph{et~al.}, ``Autonomous driving in
  urban environments: Boss and the urban challenge,'' \emph{J. Field Robot.},
  vol.~25, no.~8, pp. 425--466, 2008.

\bibitem{bai2014integrated}
H.~Bai, D.~Hsu, and W.~S. Lee, ``Integrated perception and planning in the
  continuous space: A pomdp approach,'' \emph{Int. J. Robot. Res.}, vol.~33,
  no.~9, pp. 1288--1302, 2014.

\bibitem{brechtel2014probabilistic}
S.~Brechtel, T.~Gindele, and R.~Dillmann, ``Probabilistic decision-making under
  uncertainty for autonomous driving using continuous pomdps,'' in \emph{2014
  IEEE 17th International Conference on Intelligent Transportation Systems
  (ITSC)}.\hskip 1em plus 0.5em minus 0.4em\relax IEEE, 2014, pp. 392--399.

\bibitem{cunningham2015mpdm}
A.~G. Cunningham, E.~Galceran, R.~M. Eustice, and E.~Olson, ``Mpdm: Multipolicy
  decision-making in dynamic, uncertain environments for autonomous driving,''
  in \emph{2015 IEEE International Conference on Robotics and Automation
  (ICRA)}.\hskip 1em plus 0.5em minus 0.4em\relax IEEE, 2015, pp. 1670--1677.

\bibitem{galceran2015multipolicy}
E.~Galceran, A.~G. Cunningham, R.~M. Eustice, and E.~Olson, ``Multipolicy
  decision-making for autonomous driving via changepoint-based behavior
  prediction.'' in \emph{Robotics: Science and Systems}, 2015, pp. 2290--2297.

\end{thebibliography}

\end{document}